\documentclass[letterpaper]{article} 
\usepackage[]{aaai23}  
\usepackage{times}  
\usepackage{helvet}  
\usepackage{courier}  
\usepackage[hyphens]{url}  
\usepackage{graphicx} 
\usepackage{natbib}  
\usepackage{caption} 
\usepackage{algpseudocode}
\usepackage{algorithm} 
\algnewcommand{\algorithmicgoto}{\textbf{go to}}%
\algnewcommand{\Goto}[1]{\algorithmicgoto~\ref{#1}}%
\usepackage{booktabs}
\usepackage{multirow}

\usepackage{xcolor}
\usepackage{amsmath,amsfonts,amssymb,amsthm,mathtools}

\newcommand{\DD}{\mathcal{D}}

\newcommand{\OO}{\mathcal{O}}
\newcommand{\PP}{\mathcal{P}}

\newcommand{\XX}{\mathcal{X}}
\newcommand{\YY}{\mathcal{Y}}
\newcommand{\real}{\mathbb{R}}

\newcommand{\ie}{\textit{i.e. }}

\newcommand{\bv}[1]{\boldsymbol{\mathbf{#1}}}

\theoremstyle{definition}
\newtheorem{thm}{Theorem}[section]
\newtheorem{lem}[thm]{Lemma}

\theoremstyle{definition}

\theoremstyle{remark}

\usepackage{cleveref}
\crefname{thm}{lemma}{lemmas}
\usepackage{enumitem}
\usepackage[caption=false,font=normalsize,labelfont=sf,textfont=sf]{subfig}
\usepackage{newfloat}
\usepackage{listings}
\DeclareCaptionStyle{ruled}{labelfont=normalfont,labelsep=colon,strut=off} 
\floatstyle{ruled}
\newfloat{listing}{tb}{lst}{}
\floatname{listing}{Listing}
\title{Online Verification of Deep Neural Networks \\under Domain Shift or Network Updates}
\author {
    Tianhao Wei,\textsuperscript{\rm 1}
    Changliu Liu \textsuperscript{\rm 1}
}
\affiliations {
    \textsuperscript{\rm 1} Carnegie Mellon University\\
    twei2@andrew.cmu.edu, cliu6@andrew.cmu.edu
}

\usepackage{bibentry}

\begin{document}

\maketitle

\begin{abstract}
Although neural networks are widely used, it remains challenging to formally verify the safety and robustness of neural networks in real-world applications. Existing methods are designed to verify the network before deployment, which are limited to relatively simple specifications and fixed networks. These methods are not ready to be applied to real-world problems with complex and/or dynamically changing specifications and networks. 
To effectively handle such problems,
verification needs to be performed online when these changes take place. However, it is still challenging to run existing verification algorithms online.
Our key insight is that we can leverage the temporal dependencies of these changes to accelerate the verification process.
This paper establishes a novel framework for scalable online verification to solve real-world verification problems with dynamically changing specifications and/or networks.
We propose three types of acceleration algorithms: Branch Management to reduce repetitive computation, Perturbation Tolerance to tolerate changes, and Incremental Computation to reuse previous results.
Experiment results show that our algorithms achieve up to $100\times$ acceleration, and thus show a promising way to extend neural network verification to real-world applications.
\end{abstract}

\section{Introduction}\label{sec: intro}

Neural networks are widely applied to safety/security-critical applications, such as autonomous driving~\cite{chen2015deepdriving}, flight control~\cite{ng2006autonomous}, facial recognition~\cite{meng2017identity}, and stock trading~\cite{sezer2017deep}. These applications require the network to behave as expected. We can write these expectations in mathematical specifications, and use formal verification to check whether the network satisfies the specifications. A specification is usually encoded as an input-output property: given an arbitrary input from an input set, whether the output of the network falls in a specified output set. For example, for a classification network, the input set can be: images of an apple from different angles, and the specified output set is all outputs that classify the image as ``apple". There are many existing works that can formally verify these input-output properties for neural networks~\cite{Liu_Arnon_Lazarus_Barrett_Kochenderfer_2019}.



However, applying existing neural network verification methods to real-world applications still faces many challenges. Most existing methods are designed for offline verification before the network is deployed. But in real-world applications, there are two cases that offline verification cannot handle in advance. The first case is \textit{when the specification is data-dependent and the data domain is time-varying in real time}. If we perform offline verification in this case, the offline specification must include all possible data domains, which can form an extremely large specification that may essentially cover the whole input space and is computationally intractable to verify. 
E.g., to prove the robustness of an object detector for video streaming. It is not enough to just show that the detector can tolerate certain input perturbations on images from a finite training set. Because the streaming may contain unseen images. Offline verification requires us to verify the detector on all potential images, which is impossible. 
The second case is \textit{when the neural network evolves after deployment}. In fact, many applications require the neural network to adapt online \cite{si2019agen}. For example, a meta-learned robot adapts to unseen tasks \cite{finn2017model} and a behavior prediction network adapts to subjects. If the potential online evolution is not considered, the offline verification results cannot guarantee the safety or robustness of the neural network during online execution. On the other hand, it is computationally intractable to make the offline verification cover all possible online evolution. 
New methods are desired to address these challenges.

All these challenges motivate online verification. Instead of verifying a hard and complex problem once for all, we can \textit{verify a sequence of time-varying problems on the fly to provide guarantees that the neural network is safe and robust to use now and in the near future}. The specification for these online problems can be centered around the current data distribution and current network instead of covering all possible data distributions.
E.g., we only need to verify the object detector is robust for the current image of a video streaming instead of all potential images.


Nevertheless, the trade-off introduced by turning offline verification into online verification is that the online problems need to be verified in real-time when either the specification or the network changes. But existing verification algorithms are too slow for real-time execution. On the other hand, online problems exhibit temporal dependencies (e.g., a video stream contains a sequence of correlated images, and the parameters in a network change incrementally during online adaptation). We can exploit these temporal dependencies to accelerate the verification process.
In this work, we consider two common types of temporal dependency: temporal dependency in the input data under domain shift and temporal dependency in the network parameters under network updates. We first analyze the computation bottleneck of existing verification algorithms 
and propose three principles to accelerate the verification for online problems: 
Branch Management to reduce repetitive computation, Perturbation Tolerance to tolerate changes, and Incremental Computation to reuse previous results.
Then we derive concrete algorithms based on these principles and achieve up to $100\times$ acceleration. 
Our analysis and solutions are based on reachability-based verification methods, but the core algorithms can also be applied to other verification methods. 

In summary, our contributions are three-fold: 
1) introducing online verification to greatly reduce verification difficulty for problems with time-varying specifications or time-varying networks; 
2) analyzing computation bottlenecks of existing verification algorithms and proposing three acceleration principles: Branch Management, Perturbation Tolerance, and Incremental Computation;
3) developing concrete algorithms for online verification based on these principles. The algorithms achieve up to $100\times$ speed up. 

\section{Problem formulation for Online Verification} \label{sec: formulation}
    
    This section provides a formal description of online verification problems progressively. We consider feedforwarding ReLU neural networks with input-output specifications.
    
    \subsection{Neural network and specifications}
        
        Consider an $n$-layer feedforward neural network that represents a function $\bv f$ with input $\bv{x} \in \DD_x \subset \real^{k_0}$ and $\bv{y} \in \DD_y \subset \real^{k_n}$, \ie $\bv{y} = \bv f(\bv x)$, where $k_0$ is the input dimension and $k_n$ is the output dimension. 
        Each layer in $\bv f$ corresponds to a function $\bv f_i: \real^{k_{i-1}}\mapsto \real^{k_i}$, where $k_i$ is the dimension of the hidden variable $\bv{z}_i$ in layer i. Hence, the network can be represented by
        $
            \bv f = \bv f_n \circ \bv f_{n-1} \circ \cdots \circ \bv f_1
        $. 
        The function at layer $i$ is
        $
            \bv z_i = \bv f_i(\bv z_{i-1}) = \bv{\sigma}_i(\bv W_i [\bv z_{i-1}; 1])
        $, 
        where $\bv W_i \in \real^{k_i\times k_{i-1}}$ is the weight matrix (including the bias term), and $\bv \sigma_j: \real^{k_{i}}\mapsto \real^{k_i}$ is the activation function. We only consider ReLU activation in this paper. And for simplicity, we use $\bv W$ to denote all the weights and biases of the network. 
        This paper considers input-output specifications that for all $\bv x\in \XX$, we require $\bv y = \bv f(\bv x) \in \YY$. 
        For simplicity, we assume that $\XX \subset \real^{k_0}$ is a polytope defined by $m_x$ linear constraints. And $\YY \subset \real^{k_n}$ is the expected neuraln defined by $m_y$ linear constraints. Denote $\{\bv y : \bv y=\bv f(\bv x), \forall \bv x \in \XX\} $ as $\bv f(\XX)$. Then the specification can be written as $\bv f(\XX) \subseteq \YY.$

     \subsection{Online verification} 
        
        In real applications, the input-output specification and the network weights may change with time. The verification problem for time-varying systems is defined as $\PP\{t_0, \XX(t), \bv f^t(\cdot), \YY(t)\}$, which is a tuple of the initial time, the time-varying input set, the time-varying network, and the time-varying expected output set. Since the rate of change in $\XX$, $\bv f$, and $\YY$ highly depends on the data received online, it is difficult to fully characterize these time-varying functions offline. Hence offline verification needs to ensure the specification is satisfied in the worst-case scenarios, \ie the output of the network given any possible input should comply with any output constraint: 
        $
            \bv f^{t'} \left(\bigcup_t \XX(t)\right) \subseteq \bigcap_t \YY(t),\ \forall t'.
        $
        This specification is conservative and can be intractable in many cases. 
        In contrast, if the verification is perfomed online, then at any time step $t$, the online specification is simply: 
        \begin{align}\label{eq: online verification problem}
            \bv f^t\left(\XX(t) \right) \subseteq \YY(t),\ \forall t.
        \end{align}
        We need to verify that \eqref{eq: online verification problem} holds at every time step before using the output of the network in subsequent tasks. Otherwise, we need to stop the online process and repair the network. The repairing is out of the scope of this paper and will be left for future work. In the following discussion, we assume that at any time $t$, earlier specifications are all satisfied.

        \subsection{Temporal dependencies}\label{sec: temporal}
        To run online verification of \eqref{eq: online verification problem} efficiently, we need to leverage the temporal dependencies of the online problems. We say an online verification problem has temporal dependency if the difference between two consecutive problem instances is bounded. Two common types of temporal dependency are domain shift and network updates.
        
        \textit{Domain shift} corresponds to the case that the input set changes throughout time but the network is fixed. Such as a vision module with video input, or a robot controller with nonstationary state input. We assume that only $\XX(t)$ changes with time, and $\forall t,\ \bv f^t = \bv f^{t_0},\ \YY(t) = \YY(t_0)$. 
        To measure the change of the input set, we define a set distance function (maximum nearest distance) $\Delta: \mathbb{R}^{k_1}\times\mathbb{R}^{k_2}\mapsto\mathbb{R}^+$, an input set distance function  $\Delta_{in}: \mathbb{R}\times\mathbb{R}\mapsto\mathbb{R}^+$, and the maximum input set distance $\Delta_{in}^*\in\mathbb{R}^+$: 
        \begin{align}
        \Delta(S_1, S_2)  &= \max_{\bv x' \in S_2} \min_{\bv x\in S_1} \| \bv x'- \bv x \|, \label{eq: set-dis}\\
        \Delta_{in}(t_1, t_2)  &= \Delta(\XX(t_1), \XX(t_2)),\\ 
        \Delta_{in}^* &= \max_t \Delta_{in}(t, t+1).\label{eq: Delta-in}
        \end{align}
        These functions measure the one-step change and biggest one-step change of the input set.
        
        \textit{Network updates} corresponds to the case that the network weights change but the input domain is fixed. It appears in robot learning, meta-learning, and online adaptation, such as a vehicle trajectory prediction network adapting to drivers' personalities. We assume that only $\bv W(t)$ changes with time, and $\forall t,\ \XX(t) = \XX(t_0),\ \YY(t) = \YY(t_0)$. A special case of network updates is \textit{network fine-tuning}, where only the last layer of the network $\bv W_n(t)$ changes with time. 
        To measure the change of the weights, we define a network layer-wise difference function $\bv \Delta_{layer}: \mathbb{R}\times\mathbb{R}\mapsto\mathbb{R}^+$, and a maximum one-step layer-wise difference $\bv \Delta_{layer}^*\in \mathbb{R}^+$: 
        \begin{align}
            \Delta_{layer}^i(t_1, t_2) &= \|\bv W_i(t_1) -     \bv W_i(t_2)\|_\infty,\\
            \bv \Delta_{layer}(t_1, t_2) &= \left\{ \Delta_{layer}^i(t_1, t_2) \right\}_i,\\
            \bv \Delta_{layer}^* &= \max_t \bv \Delta_{layer}(t, t+1).\label{eq: Delta_net}
        \end{align}

        
        
        
        


\subsection{Related work}
To the best knowledge of the authors, all previous works assume static data distribution and fixed networks.
These works can be broadly categorized into optimization-based methods and reachability-based methods. Optimization-based methods try to falsify the input-output specification by primal~\cite{Tjeng_Xiao_Tedrake_2019} or dual~\cite{wong2018provable} optimization. And reachability-based methods \cite{tran2020nnv, Gehr_Mirman_Drachsler-Cohen_Tsankov_Chaudhuri_Vechev_2018} perform layer-by-layer reachability analysis to compute the reachable set from the input set. Existing methods include MaxSens~\cite{Xiang_Tran_Yang_Johnson_2020}, ReluVal~\cite{wang2018formal}, and Neurify~\cite{wang2018efficient}. Both types can be combined with branching methods (called search-based methods in  \cite{Liu_Arnon_Lazarus_Barrett_Kochenderfer_2019}) to improve efficiency and accuracy. Branching methods partition the input domain or the function space into many branches, and verify them one by one to reduce over-approximation. In this work, we focus on reachability plus branching methods, while core methodology can also be applied to other verification methods. 



\begin{figure}[tbh]
    \centering
    \includegraphics[scale=0.6]{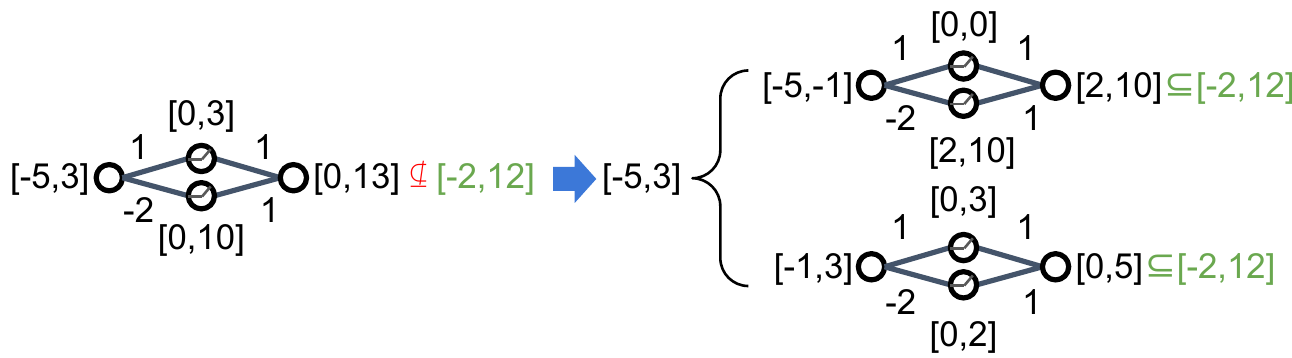}
    \caption{Reachability plus branching method on a simple neural network with one hidden layer and ReLU activation. The reachable set is computed by interval arithmetic \cite{wang2018formal}, which propagates the lower bound and upper bound of neurons layer-by-layer. The input set is $[-5,3]$, and the expected output set is $[-2,12]$. If we directly compute the reachable set, the approximated reachable set is $[0,13]$, which is not a subset of the expected output set. But if we split the input set into two branches $[-5,-1]$ and $[-1,3]$ and compute the reachable sets. Then the two reachable sets $[2,10]$ and $[0,5]$ are both subsets of the expected output set.}
    \label{fig: reach example}
\end{figure} 

\section{Online verification}
    
    The key idea for efficient online verification is to leverage the temporal dependencies to speed up the computation. 
    The accelerated verification algorithm is called an \textit{online verification algorithm}.
    This paper develops online verification algorithms based on reachability plus branching methods. 
    In the following discussion, we first review the procedures of reachability plus branching methods. Then discuss the steps that we can accelerate and propose three acceleration principles. Finally, we propose detailed algorithms for domain shift and network updates based on these principles.
    
    \subsection{Base Algorithm: Reachability plus branching}
    
        \begin{algorithm}[htb]
            \small
        	\caption{Reachability plus branching algorithm 
         }
        	\begin{algorithmic}[1]
                \Function{reach + branch}{$\XX,\bv f, \YY$}
                        \State $Q \leftarrow \{\XX_0=\XX\}$ \Comment{Branches to split}
                        \State $X_S$ $ \leftarrow \{\}$\Comment{Final branches}
                        \While{$Q \neq \emptyset$ \textbf{and} Not exceeding computation limit} 
                        \State $\XX_i \leftarrow$ pop front($Q$)
                        \State $\OO_i \leftarrow$ \Call{Reach}{$\XX_i, \bv f$}
                        \State result$_i\leftarrow$ \Call{Check}{$\OO_i,\YY$} 
                        \If {result$_i = $ ``Hold"}     \Call{Push}{$X_S$, $\XX_i$}
                        \ElsIf{result$_i = $ ``Unknown"}
                            \State $\XX_{2i}$, $\XX_{2i+1}$ = \Call{Split}{$\XX_i$}
                            \State \Call{Push back}{$Q$, $\XX_{2i}$, $\XX_{2i+1}$}
                        \ElsIf {result$_i = $ ``Violated"}
                            \Return ``Need repair"
                        \EndIf
                        \EndWhile
                        \State \Return ``Hold", $X_S$\ \  \textbf{if}\ \  {$Q = \emptyset$} \ \ \textbf{else}\ \ ``Unknown", $X_S+Q$
            \EndFunction
        	\end{algorithmic} \label{pseudo: rpb}
        \end{algorithm}

      
        It is difficult to compute $\bv f(\XX)$ because of the high nonlinearity of neural networks. Reachability based methods use linear relaxation to compute over-approximated output set $\OO(\XX, \bv f) \supseteq \bv f(\XX)$.
        Then the specification holds if $\OO(\XX, \bv f) \subseteq \YY$.
        But, if $\OO(\XX, \bv f) \nsubseteq \YY$, it is still possible that $\bv f(\XX) \subseteq \YY$. In this case, the verification algorithm can not decide whether the specification holds or not. Further refinement is required. 
        Then the reachability plush branching algorithm reduces the overapproximation by recursively splitting $\XX$ into many subsets $\XX_i$, which we call branches, and computes the reachable set $\OO(\XX_i, \bv f)$ for each branch and checks their compliance with the output constraints. The over-approximation reduces as the input set gets smaller because more ReLU nodes can be reduced into a linear function, hence reducing the nonlinearity of the network. An example is shown in \cref{fig: reach example}. The final set of branches is denoted by $X_S$. We expect all the branches in $X_S$ are verified safe. But sometimes due to time or memory limits, $X_S$ may contain unknown branches. In this case, we need more computation resources to achieve online verification. And if there are violated branches, we need to stop the verification process and repair the network. 
        The pseudocode is shown in \cref{pseudo: rpb}. 
    
    
        

    \subsection{Acceleration principles}
    The most time-consuming part of verification is the massive reachable set computation caused by iterative splitting~\cite{bak2021second}, especially in the online verification case because we have to call the Reach + Branch function every time step. It takes up to $99.4\%$ time in our experiment.
    Our goal is to reduce iterative splitting and accelerate reachable set computation by temporal dependencies.
    
    We first reduce iterative splitting by reusing the branching of last time step. That is, we skip the iterative splitting and directly update and check the final branches $X_S$ of the last step. 
    We propose a principle named \textit{branch management}, which keeps the branches untouched as long as possible while maintaining verification accuracy. 
    There is a trade-off between accuracy and efficiency: when the setting changes, reconstructing the branches gives us the most accurate results but is time-consuming. Branch management aims to find a balance between accuracy and speed. 
    
    Then we accelerate reachable set computation with different strategies based on the status of a branch: 
    1. When a branch is unchanged, that is $\XX_i(t)=\XX_i(t-1)$ and $\bv f^t=\bv f^{t-1}$. We can directly reuse the previous result;
    2. When a branch is changed, either $\XX_i(t)\neq\XX_i(t-1)$ or $\bv f^t \neq \bv f^{t-1}$, we wish to directly infer the verification results by checking whether a previously verified robustness margin can tolerate the current changes.
    This leads to the \textit{perturbation tolerance} principle. 
    Formally, we want to establish a $\delta$ robust margin for a time step $t_0$. 
    For domain shift, we want to find a $\delta$ such that: $\Delta_{in}(t, t_0) < \delta \implies \OO_i(t) \subseteq \OO_i(t_0) \subseteq \YY$, where $\OO_i(t)$ is short for $\OO(\XX_i(t), \bv f^t)$.     And for network updates, we want to find a $\bv\delta$ such that: $\bv \Delta_{layer}(t,t_0) < \bv \delta \implies \OO_i(t) \subseteq \OO_i(t_0) \subseteq \YY$. 
    Then we can infer the verification result without computing the reachable sets.
    3. If the direct inference fails, we have to compute $\OO_i(t)$. We wish to accelerate the computation using previous results. This leads to the \textit{incremental computation} principle.
    For reachability algorithms, we can reuse the reachable sets of unchanged layers. 
      
        \begin{algorithm}
        	\caption{Online verification algorithm}
        	\small
        	\begin{algorithmic}[1]
        	\Function{online verification}{$t, \XX(t),\bv f^t, \YY(t), t_0$}
            	    \State $X_S(t),\text{tag} \leftarrow$ \Call{BranchManage}{$\XX(t)$, $\bv f^t$, $X_S(t-1)$} \hfill
            	    \For {($i, \XX_i$) in enumerate($X_S(t)$)}\label{step: check_branch}
            	    \If{$\text{tag}_i = $ reuse} 
            	        \State $\OO_i(t) \leftarrow \OO_i(t-1)$ \Comment{Reuse previous results.}
            	    \ElsIf {\Call{Tolerable}{$t, t_0$}}
            	        \State $\OO_i(t) \leftarrow \OO_i(t_0)$ \Comment{Change is tolerable}
            	    \Else 
            	        \State $t_0 \leftarrow t$ \Comment{Update the robust margin}
            	        \State $\OO_i(t) \leftarrow$ \Call{Reach}{$\XX_i(t), \bv f^t$} 
            	    \EndIf
            	    \State result$_i\leftarrow$\Call{Check}{$\OO_i(t),\YY(t)$} 
            	\EndFor
            \EndFunction
        	\end{algorithmic} \label{pseudo: online}
        \end{algorithm}
        
    
    In the following discussion, we derive detailed acceleration algorithms by applying these principles to the two kinds of temporal dependencies: domain shift and network updates. The pseudocode for these acceleration algorithms is shown in \cref{appendix: pseudocode}.
    These acceleration algorithms combined with reachability plus branching method form a framework of online verification algorithm, as shown in \cref{pseudo: online}. 

\subsection{Acceleration algorithms for domain shift}\label{sec: domain-shift}
    We derive three acceleration algorithms for domain shift (where only input changes with time): branch management for input, reachable set relaxation, and Lipschitz bound. \Cref{fig: domain shift} shows examples of applying these algorithms to accelerate the computation. 
    
    We assume the input set $\XX(t)$ is defined by $m_x$ time-varying linear constraints 
    $\bv A^b(t) \bv x \leq \bv b^b(t)$,
    which we call \textit{base constraints}. We assume 
    $\bv A^b(t)$ and $\bv b^b(t)$
    change gradually. The branching algorithm splits the input set $\XX(t)$ by adding more \textit{splitting constraints} $\bv A^s_i(t), \bv b^s_i(t)$ for branch $i$. And we denote all the constraints of $\XX_i$ by $\bv A_i, \bv b_i$.
        \begin{align}
            \XX(t) &= \{\bv x \mid \bv A^b(t) \bv x \leq \bv b^b(t)\}\\
            \XX_i(t) & = \{\bv x \mid \bv A^b(t) \bv x \leq \bv b^b(t),\ \bv A^s_i(t) \bv x \leq \bv b^s_i(t)\}\\
            &= \{\bv x \mid \bv A_i(t) \bv x \leq \bv b_i(t) \}.
        \end{align}
        
    \subsubsection{Branch Management for Input (BMI)}\label{tech: branch_management}
         Following the branch management principle, when the input set changes, we reuse the previous branching, recompute the reachable sets only when necessary, and dynamically maintain the branches to maximize accuracy.
         
        BMI leverages the fact that the branching does not change drastically when $\XX(t)$ changes gradually, that is, the splitting constraints $\bv A^s_i(t), \bv b^s_i(t)$ do not change a lot. The splitting constraints $\bv A^s_i(t), \bv b^s_i(t)$ are chosen based on heuristics that depends on $\XX(t)$~\cite{wang2018efficient}. When $\XX(t)$ changes with time, ideally we need to update $\bv A^s_i(t), \bv b^s_i(t)$ by iterative splitting to minimize the over-approximation. But it is time-consuming. Therefore, we reuse previous branching by letting $X_S(t) : = X_S(t-1)$, then for each $\XX_i(t) \in X_S(t)$, we only update the base constraints from $\bv A^b(t-1), \bv b^b(t-1)$ to $\bv A^b(t), \bv b^b(t)$ and leave $\bv A^s_i(t-1), \bv b^s_i(t-1)$ untouched. And when this strategy leads to a significant increase in over-approximation, we reconstruct $X_S(t)$ by reachability plus branching. The over-approximation can be measured in multiple ways, one way is to use the coverage rate (to be defined in \cref{sec: exp}).
        
        Besides reusing branching, BMI only computes the reachable set $\OO_i(t)$ when necessary, i.e., when $\XX_i(t) \nsubseteq \XX_i(t-1)$: 
        \begin{align}
            \max_{\bv x \in \{\bv x \mid \bv A_i(t) \bv x - \bv b_i(t) \leq 0\}}\ & \bv A_i(t-1) \bv x - \bv b_i(t-1) > 0 \label{eq: BMI}
        \end{align}
        Because when $\XX_i(t)$ is unchanged or shrinking, the previous verification result directly holds.
        
        
        The pseudo-code is shown in \cref{algo: BMI}. This algorithm does not increase over-approximation. It is worth noting that BMI works best for a few changing constraints and slow changes in order to leverage the temporal dependency.

        \begin{figure}[t]
            \centering
            \includegraphics[width=\linewidth]{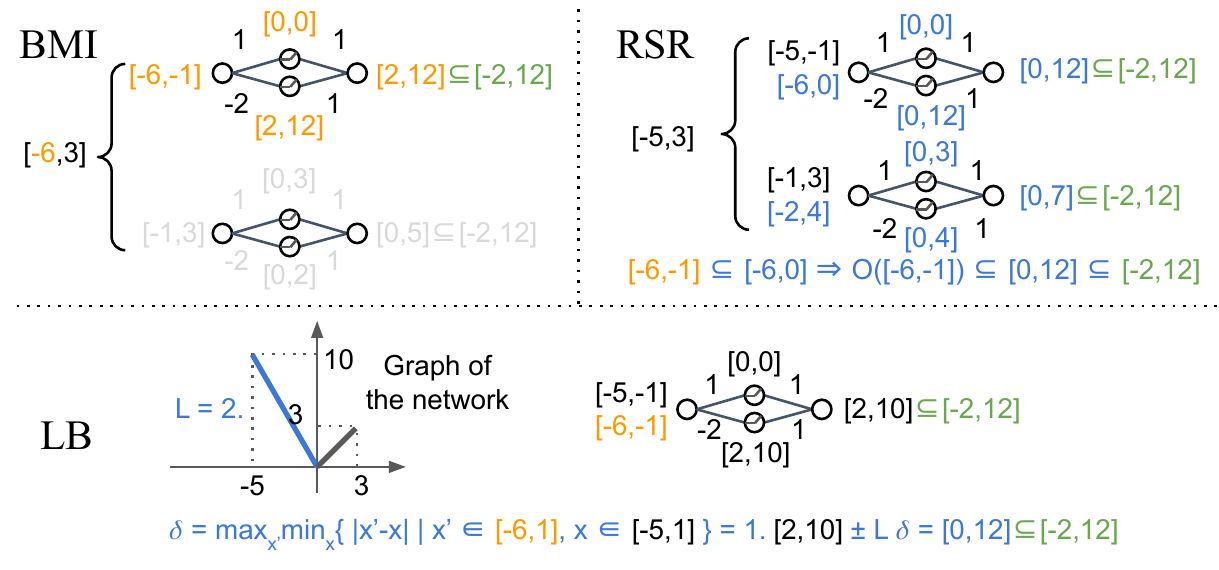}
            \caption{Online verification algorithms to address domain shift. Orange denotes input set and reachable set change, gray denotes unchanged branches, and blue denotes algorithm operation. (a) BMI. When the input set changes from $[-5,3]$ to $[-6,3]$, we reuse the branching and only the first branch requires re-computation. Therefore, we can skip the computation of the second branch. (b) RSR. When we compute the reachable set of $[-5,-1]$, we relax the input set from $[-5,-1]$ to $[-6,0]$ to tolerate potential input perturbations. The corresponding reachable set is $[0,12]$. When the input set changes to $[-6,-1]$, we first check that it is a subset of the relaxed input set $[-6,0]$. Then we can directly assert that the reachable set must be a subset of the $[0,12]$. (c) LB. The Lipschitz constant of the network is $2$. The distance between the new input set and the original input set is 1. With the original reachable set as $[2,10]$, we can assert that the new reachable set must be a subset of $[2,10] \pm 2 \cdot 1 = [0,12]$.}
            \label{fig: domain shift}
        \end{figure}

    
    \subsubsection{Reachable set relaxation (RSR)} \label{sec}
        
        RSR follows the perturbation tolerance principle. Once a large input set is verified safe, any subset of this input set is inherently safe. Therefore we can skip the computation.
        
        Specifically, for a branch $\XX_i(t_0) = \{x \mid \bv A_i(t_0) \bv x \leq \bv b_i(t_0)\}$, we add an offset $\bv \Delta$ to create an enlarged input set: $\hat\XX_i(t_0) = \{x \mid \bv A_i(t_0) \bv x \leq \bv b_i(t_0) + \bv \Delta\}$.
        Then if the corresponding relaxed reachable set $\hat \OO_i(t_0) \subseteq \YY$. Then we can skip the computation for future $\XX_i(t)$ if $\XX_i(t) \subseteq \hat\XX_i(t_0)$, that is, when $\max_{\bv x\in\XX_i(t)} \bv A_i(t_0) \bv x - \bv b_i(t_0) - \bv \Delta \leq 0$,
        which is a linear programming problem. \cref{eq: BMI} can be viewed as a special case of RSR where $\bv \Delta = 0$ and without over-approximation.
        
        A larger offset $\bv \Delta$ enables larger tolerance to perturbations, therefore potentially reducing more computation.  But since the input set is enlarged, the over-approximation is increased. We can only relax the reachable set up to the point that we still have a confirmative result for that branch. 
        We draw a trade-off curve to guide the user to choose the appropriate offset in \cref{sec: trade-off}.
        
        
        One challenge of RSR is that constructing the relaxed reachable sets at step $t_0$ can be time-consuming. Therefore, we propose to ensure real-time computation by multiprocessing. We can compute $\hat\XX_i(t_1)$ in the background while using $\hat\XX_i(t_0)$ to skip the verification of $\XX_i(t_1)$. Then we can use $\hat\XX_i(t_1)$ to skip the verification of $\XX_i(t_2)$, etc.
        We provide a formal guarantee for real-time computation in \cref{appendix: RSR-real-time}.
        
    \begin{figure}[t]
        \centering
        \includegraphics[scale=0.68]{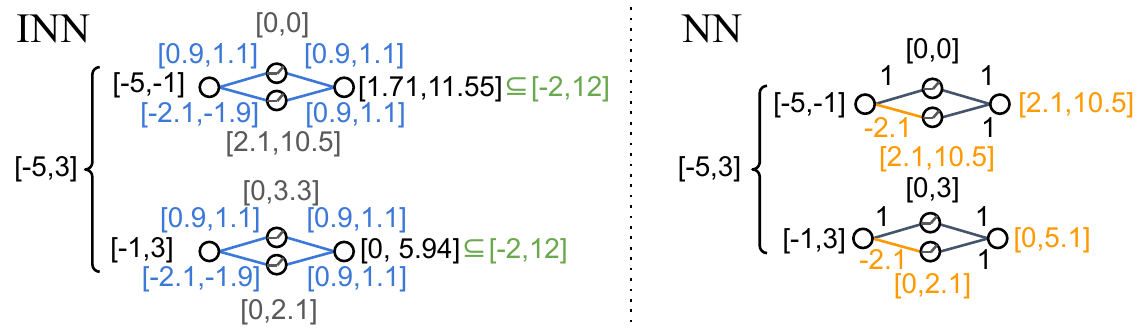}
        \caption{Online verification algorithms to address network updates. Orange denotes network weights change and blue denotes algorithm operation. The left figure shows an INN that is constructed at $t_0$. Each weight is an interval. The right figure shows the NN at $t_1$ when one weight changes from $2$ to $2.1$. With BMW, we keep the branching unchanged. And with INN, because $-2.1 \in [-2.1,-1.9]$, the NN is covered by the INN. 
        We can assert that the reachable set at $t_1$ must be a subset of the reachable set of the interval network. Therefore no additional computation is needed.}
        \label{fig: weights-shift}
    \end{figure}

    \subsubsection{Lipschitz Bound (LB)}
        Following the perturbation tolerance principle, LB provides a condition to skip the computation by using the Lipschitz continuity of the neural network.
        For a branch $\XX_i(t)$ and linear output constraints $\YY: = \{ \bv y \mid \bv C \bv y < \bv d\}$. We have the following lemma: 
        \begin{lem}(Lipschitz Tolerance)
        Suppose $\bv f(\XX_i(t_0)) \subseteq \YY$, then $\bv f(\XX_i(t)) \subseteq \YY$ if 
        \begin{align}
            \Delta_{in}(t_0, t) \leq \min_j \min_{\bv y \in \OO(\XX_i(t_0),\bv f)} \frac{ d_j - \bv c_j^T\ \bv y}{\|\bv c_j^T\|\ L}. \label{eq: lem_eq}
        \end{align}
        where $\Delta_{in}(t_0, t)$ is defined in \eqref{eq: Delta-in}, $\bv c_j$ is the $j^{th}$ row of $\bv C$, and $d_j$ is the $j^{th}$ value of $\bv d$, and $L$ is the Lipschitz constant of the neural network. Proof can be found in \cref{sec: proof-lem}.
        \end{lem}
        
        The RHS of \eqref{eq: lem_eq} can be computed at time step $t_0$ when we check whether $\OO_i(t_0) \subseteq \YY$ without additional time cost. If $\Delta(t_0, t)$ is unknown, we can get it by solving \cref{eq: set-dis}, which is a convex optimization because $\XX_i(t_0)$ and $\XX_i(t)$ are both polytopes (proof is in \cref{sec: proof-lem}). The Lipschitz constant $L$ can be computed offline in advance by semidefinite programming~\cite{fazlyab2019efficient} because the network does not change. Therefore, LB can reduce the computation time without increasing over-approximation.
        
    
\subsection{Acceleration algorithms for network updates}\label{sec: weight-shift}
     We introduce two algorithms to address network updates: branch management for weight and interval neural networks. An example is shown in \cref{fig: weights-shift}. We also propose an algorithm, incremental computation, specifically for the fine-tuning case where only the last layer of the network changes. 
    
    \subsubsection{Branch management for weight (BMW)}
     Following the branch management principle, when the network changes, we directly reuse the previous branching by letting $X_s(t): = X_s(t-1)$ and see whether the over-approximation increases (measured by coverage rate to be defined in \cref{sec: exp}). If so, we re-construct the branches. BMW works the best if the branching heuristic does not depend on the network, such as divide by dimensions \cite{Liu_Arnon_Lazarus_Barrett_Kochenderfer_2019}. In this case, we do not need to reconstruct the branches at all.
     
    \subsubsection{Interval neural network (INN)}\label{sec: INN}
        INN was proposed to abstract NN by merging neurons~\cite{prabhakar2020abstraction}. But following the perturbation tolerance principle, we propose to construct an INN to resist network updates.
        In an INN, the weights $\bv W$ are replaced by an interval $[\underline{\bv W}, \overline{\bv W}]$ bounded by a lower bound $\underline{\bv W}$ and an upper bound $\overline{\bv W}$. And the same for $\bv b$. Then the propagation through a layer can be expressed by interval arithmetic: 
        \begin{align}
            [\underline{\bv z_i}, \overline{\bv z_i}] = \bv \sigma_i([\underline{\bv W_i}, \overline{\bv W_i}][\underline{\bv z_{i-1}}, \overline{\bv z_{i-1}}] + [\overline{\bv b_i}, \underline{\bv b_i}])
        \end{align}
        As long as the weights change within the interval, we can directly reuse the reachable set of this interval network. When the network change exceeds the interval range, we re-construct the interval network based on the current weights. 
        
        Computing the reachable set of an interval neural network does not introduce extra time costs than ordinary networks. But over-approximation is amplified by the interval. When the interval is larger, the verification result can tolerate larger perturbations but also over-approximates more. We can only use an interval that still leads to a confirmative verification result. We draw a trade-off curve to choose the proper interval in \cref{sec: trade-off}.

        Similar to RSR, constructing INN at step $t_0$ can be time-consuming. Therefore, we also use multiprocessing to ensure real-time computation and provide a formal guarantee in \cref{appendix: INN-real-time}.

    \subsubsection{Incremental computation (IC)}
        
        
        Following the incremental computation principle, we derive this algorithm specifically for network fine-tuning scenarios. 
        Since only the last layer of the network changes, we can compute the reachable sets from the last unchanged layer. This method requires storing the reachable sets of the last unchanged layer.

\section{Numerical Studies} \label{sec: exp}
        
        The numerical studies are designed to answer the following questions: 
        1) How much can the algorithms accelerate under domain shift and network updates?
        2) How do the algorithms generalize to different settings?
        3) For those methods trade accuracy for speed, what is the best trade-off point? 
        
        To answer these questions, we first do ablation studies on two tasks involving real-time change. Then we analyze how the algorithms perform with respect to 5 different variables to show the generalizability. In the end, we analyze how INN and RSR trade accuracy for speed and plot their trade-off curves to find the best-balanced point. We implement the algorithms on top of NeuralVerification.jl~\cite{Liu_Arnon_Lazarus_Barrett_Kochenderfer_2019}. 
        
        Because some methods trade accuracy for speed, we first define a metric of over-approximation for fair comparisons. It's difficult to compute the over-approximation precisely. Alternatively, we use the verified input set \textbf{coverage rate}, which is the ratio of the volume of the verified input set to the volume of the total input set. Because when an algorithm over-approximates more, fewer branches are likely to have confirmative verification results. The coverage rate reduces. We use a sampling-based method to compute the coverage rate because it is challenging to measure high dimensional volumes. We randomly sample $N$ points from the input set $\XX$. A point is considered verified if it belongs to a verified branch $\XX_i$. We define the coverage rate as: $N_{\text{verified}} / N$. The higher the coverage rate is, the less over-approximation an algorithm has under the assumption that the specification should hold.

        \begin{figure*}[t]
            \begin{minipage}{\linewidth}
            \centering
            {\label{main: b}\includegraphics[width=\linewidth]{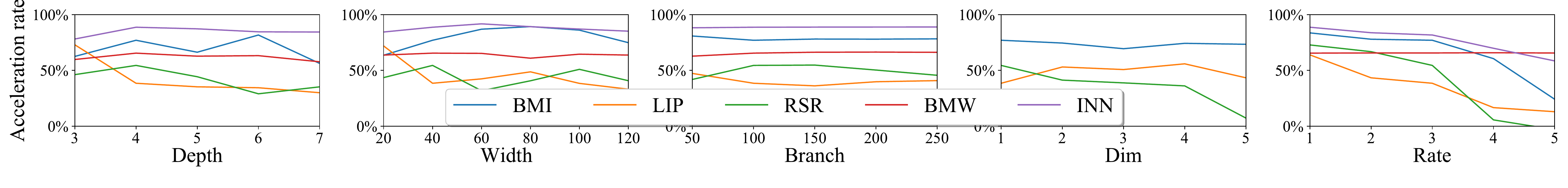}}
            \end{minipage}
            \caption{Algorithms acceleration rates with regard to five variables: network depth, network width, number of branches, number of changing input dimensions, and changing rate. The algorithms show generally good scalability.}
            \label{fig: analysis}
        \end{figure*}

        \subsection{Ablation study} 
        We demonstrate the acceleration effect on a robotics task and a CV task separately.
        
        \subsubsection{Robotics task}  \label{sec: exp-domain-shift}
        Many robotics tasks require real-time changing specifications and/or adapting networks, especially when human is involved~\cite{liu2017designing}.
        Here we consider a human motion prediction task as a representation~\cite{cheng2019human}. 
        In this task, the neural network predicts velocity vectors $\hat{\bv y} = [\hat{\bv v}_{t}, \hat{\bv v}_{t+1}, \hat{\bv v}_{t+2}]$ based on historical velocities $\bv x = [\bv v_{t-3}, \bv v_{t-2}, \bv v_{t-1}]$ , where $\bv v,\ \hat{\bv v}  \in \real^3$. 
        The range of the velocity may change with time and the network may adapt online.
        We want to make sure the predicted velocity and acceleration are always within a reasonable range given the input set.
        The static input set and expected output set are defined as: 
        \begin{align*}
            \XX &=\{ &[\bv v_{t-3}, \bv v_{t-2}, \bv v_{t-1}]\mid \|\bv v_{t-i}\| &< v_x, i=1,2,3; \\
            & &\|\bv v_{t-j} - \bv v_{t-j-1}\| &< a_x, j=1,2 \},\\
            \YY &=\{ &[\hat{\bv v}_{t}, \hat{\bv v}_{t+1}, \hat{\bv v}_{t+2}]\mid \|\hat{\bv v}_{t+i}\| &< v_y, i=0,1,2; \\
            & &\|\hat{\bv v}_{t+j+1} - \hat{\bv v}_{t+j}\| &< a_y, j=0,1 \}
        \end{align*}
        where $\|\cdot\|$ is $l_\infty$ norm, 
        $v_x$, $a_x$, $v_y$ and $a_y$  are constants.
        The following experiments are tested on $T = 100$ consecutive time steps for a 4-layer neural network with 100 neurons per hidden layer. The total number of branches is $m=250$, and the constants are $v_x=1$, $a_x=0.1$, $v_y=5$, $a_y=10$.
        
        First, we test the domain shift case.
        The domain shift is characterized by adding an constraint $\bv |\bv v_{t-1}| < [v_x, v_x, v_x - 10^{-3} \cdot (T - t)]^T$ to the input set $\XX$. 
        As shown in \cref{tab: ablation}, compared to Reach + Branch, BMI reduces the verification time to about $1/3$, and LB further reduces the time to $1/10$ without loss of accuracy. With RSR ($\bv \Delta=1\times10^{-3}$), the time reduces more, but the accuracy drops a little. 
        

        \begin{table}[ht]
            \small
            \begin{center}
                \begin{tabular}{c l  r  c}
                \toprule
                Scenario & Method & Time(s) &  Coverage\\ 
                \midrule
                \multirow{5}{*}{Domain Shift} & None & 11.234  & 99.93\% \\
                & BMI & 2.439 & 100.0\%  \\
                & BMI+LB & 1.443 & 100.0\% \\
                & BMI+RSR & 1.326 & 97.06\% \\
                & BMI+LB+RSR & 0.954 & 97.01\% \\
                \midrule
                \multirow{3}{*}{Network Updates}& None & 11.55 & 93.70\% \\
                & BMW & 3.904 & 93.67\%\\
                & BMW+INN & 0.434 & 93.14\%\\
                \midrule
                \multirow{5}{*}{Fine-tuning}& None & 11.402& 94.38\%  \\
                & BMW & 3.944  & 93.69\% \\
                & BMW+INN & 0.375 & 93.67\% \\
                & BMW+IC & 0.469 & 93.69\% \\
                & BMW+INN+IC & 0.121 & 93.67\% \\
                \bottomrule
                \end{tabular}
            \end{center}
            \caption{\label{tab: ablation} Average computation time and average coverage rate for the RL task.}
        \end{table}

        Then we test the network updates case.
        At each time step, the network weights are updated by backpropagation from an $l_2$-loss function $l = \|\bv y - \hat{\bv y}\|_2$ with a learning rate $0.001$.  $\|\bv \Delta_{layer}^*\|_{\infty}$ is around $5\times10^{-4}$. We test BMW and INN. INN uses an interval range of $5\times \bv \Delta_{layer}^*$. 
        As shown in \cref{tab: ablation}. BMW reduces the time to about $1/3$, and INN further reduces the time to about $1/20$ without significant loss of accuracy. INN trades accuracy for speed, but because the interval range is pretty small in this experiment, we only see a slight coverage rate drop. 


        We also did an additional experiment for the fine-tuning case. The update rule and learning rate are the same as above. But we only update the last layer. Besides BMW and INN, we also test IC.  
        As shown in \cref{tab: ablation}. 
        With all these algorithms, the time is reduced to about $1/100$. 
            
        

        \subsubsection{CV task} \label{sec: MNIST}
        
        \begin{figure}[ht]
            \centering
            \includegraphics[width=.3\linewidth]{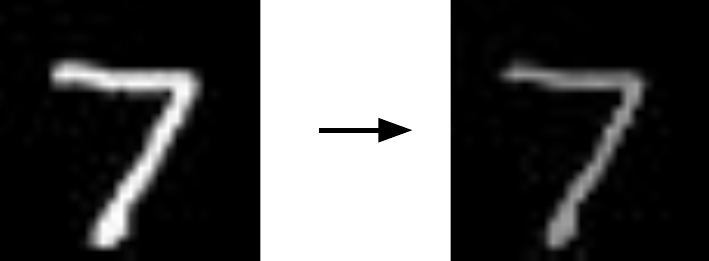}
            \caption{The image is dimming. We want to verify that the classifier is robust to perturbations for every timestep.}
            \label{fig: brightness_change}
        \end{figure}
     
        To study the possibility of applying online verification algorithms to real-world applications, such as verification of videos. We show the effect and limitations of the algorithms for image changes. We verify an image of changing brightness as shown in \cref{fig: brightness_change}. The value of each pixel decreases $1/256$ each time step. We want to continuously verify that a classifier is robust to perturbations within $2/256$ $l_\infty$-distance. 
        Because it is difficult to train a robust network that resists such large domain shifts. The verification algorithms can easily find violations of the specification on some branches. But LB and RSR both rely on the assumption that a branch is previously verified safe. Therefore we can only show the effect of BMI for domain shift. For the network updates case, BMW and INN work as expected. INN uses an interval of $5\times \bv \Delta_{layer}^*$. As shown in \cref{tab: mnist}, the algorithms still accelerate the verification in high-dimensional applications such as image classification.
        

         \begin{table}[ht]
            \centering
            \small
            \begin{tabular}{c c c}
            \toprule
                 Scenario & Method & Time (s) \\
                 \midrule
                 \multirow{2}{*}{Domain Shift} & None & 64.014 \\
                 & BMI & 36.885 \\
                 \midrule
                 \multirow{3}{*}{Network Updates} & None & 74.773 \\
                 & BMW & 24.258 \\
                 & BMW + INN & 7.788 \\
                \midrule
            \end{tabular}
            \caption{Average computation time for the CV task.}
            \label{tab: mnist}
        \end{table}

        \begin{figure}[ht]
            \centering
            \includegraphics[width=\linewidth]{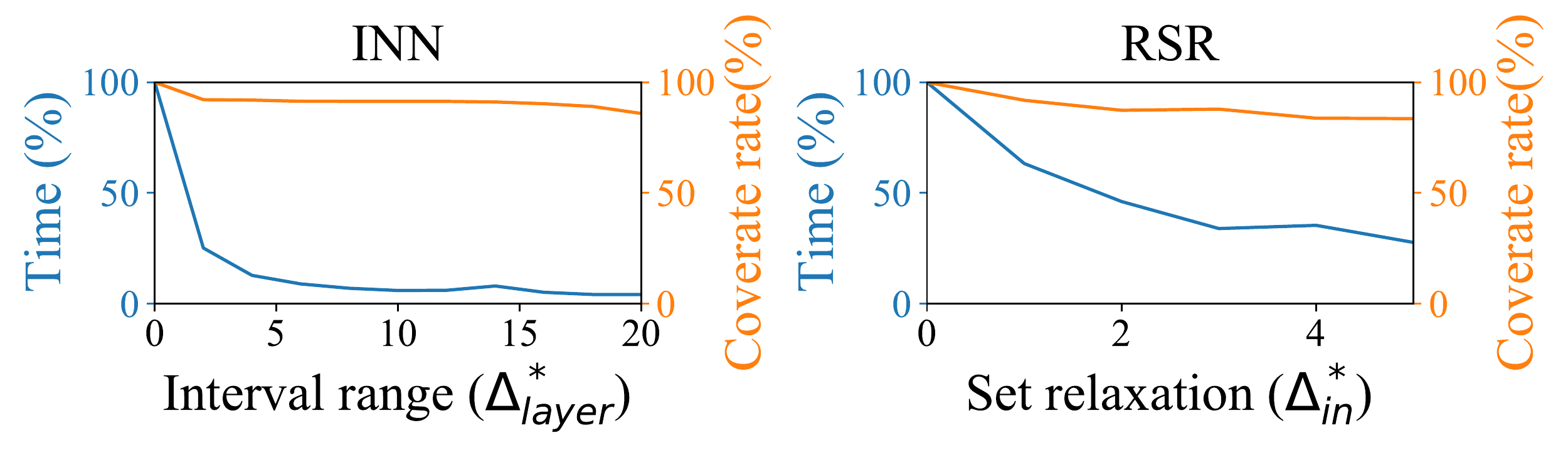}
            \caption{INN and RSR trade-off curves. The curve shows how time and coverage rate drops in percentage with the interval range or set relaxation increases. 
            }
            \label{fig: INN-RSR}
        \end{figure}

        \subsection{Scalability analysis}
        
        In this section, we study how the acceleration rate of an algorithm changes with different variables. The acceleration rate is defined as the increased percent of verification speed: $T_0 / T_1 - 1$, where $T_0$ and $T_1$ are the computation time without and with the algorithm.  We consider five variables: the depth of the network, the width of the network, the number of branches, the changing dimensions of the input, and the changing rate of input and weights. The default task setting is the same as in the RL task. We only change the controlled variables. For depth, we test networks from 3-layer to 7-layer. For width, we test hidden dimensions from 20 to 120. For the number of branches, we test from 50 to 250. For the number of changing dimensions, we test from 1 to 5, out of 9 input dimensions. For changing rate, We amplify the original change by a factor from 1 to 5. The results shown in \cref{fig: analysis} demonstrate a good generality with most variables except the changing rate, which is expected. Detailed analysis is in \cref{appendix: scalability}

        \subsection{Trade-off analysis} \label{sec: trade-off}
        
        Interval Neural Network (INN) and Reachable Set Relaxation (RSR) trade accuracy for speed. To help the user find proper hyper-parameters, we plot the trade-off curves between accuracy and speed.
        As shown in \cref{fig: INN-RSR}. We consider $5\times \bv \Delta_{layer}^*$ a well-balanced point for INN because the coverage rate has only an unnoticeable drop while the verification time reduces a lot. 
        And for RSR, time drops like a log function, and the coverage rate drops about linearly. The user can choose the relaxation offset based on their resources.
        The trade-off is also affected by the number of branches. We did an additional analysis in \cref{appendix: trade-off} to show the generalizability of the conclusion.

\section{Discussion and Limitation}\label{sec: discuss}

In this work, we proposed online verification as a potential solution to verify neural networks in real-world applications. And we studied how to leverage temporal dependencies to accelerate existing verification algorithms to a real-time level to achieve online verification. We proposed a framework to design acceleration algorithms and derived several concrete algorithms based on it. 
In the future, we will extend our methods to more types of temporal dependencies, study how to design the best acceleration algorithms, and select optimal acceleration algorithms automatically. 

Besides the verification part, training and repairing are also crucial to applying neural networks to real-world applications. Existing networks can be easily attacked by online changes. We will also study how to train online-change robust networks and how to repair networks online based on the online verification results.



\bibliography{references}  

\begin{thebibliography}{19}
\providecommand{\natexlab}[1]{#1}

\bibitem[{Bak, Liu, and Johnson(2021)}]{bak2021second}
Bak, S.; Liu, C.; and Johnson, T. 2021.
\newblock The second international verification of neural networks competition
  (vnn-comp 2021): Summary and results.
\newblock \emph{arXiv preprint arXiv:2109.00498}.

\bibitem[{Chen et~al.(2015)Chen, Seff, Kornhauser, and
  Xiao}]{chen2015deepdriving}
Chen, C.; Seff, A.; Kornhauser, A.; and Xiao, J. 2015.
\newblock Deepdriving: Learning affordance for direct perception in autonomous
  driving.
\newblock In \emph{Proceedings of the IEEE international conference on computer
  vision}, 2722--2730.

\bibitem[{Cheng et~al.(2019)Cheng, Zhao, Liu, and Tomizuka}]{cheng2019human}
Cheng, Y.; Zhao, W.; Liu, C.; and Tomizuka, M. 2019.
\newblock Human motion prediction using semi-adaptable neural networks.
\newblock In \emph{2019 American Control Conference (ACC)}, 4884--4890. IEEE.

\bibitem[{Fazlyab et~al.(2019)Fazlyab, Robey, Hassani, Morari, and
  Pappas}]{fazlyab2019efficient}
Fazlyab, M.; Robey, A.; Hassani, H.; Morari, M.; and Pappas, G.~J. 2019.
\newblock Efficient and Accurate Estimation of Lipschitz Constants for Deep
  Neural Networks.
\newblock In \emph{Advances in Neural Information Processing Systems
  (NeurIPS)}.

\bibitem[{Finn, Abbeel, and Levine(2017)}]{finn2017model}
Finn, C.; Abbeel, P.; and Levine, S. 2017.
\newblock Model-agnostic meta-learning for fast adaptation of deep networks.
\newblock In \emph{International Conference on Machine Learning}, 1126--1135.
  PMLR.

\bibitem[{Gehr et~al.(2018)Gehr, Mirman, Drachsler-Cohen, Tsankov, Chaudhuri,
  and Vechev}]{Gehr_Mirman_Drachsler-Cohen_Tsankov_Chaudhuri_Vechev_2018}
Gehr, T.; Mirman, M.; Drachsler-Cohen, D.; Tsankov, P.; Chaudhuri, S.; and
  Vechev, M. 2018.
\newblock AI2: Safety and Robustness Certification of Neural Networks with
  Abstract Interpretation.
\newblock In \emph{2018 IEEE Symposium on Security and Privacy (SP)}, 3–18.

\bibitem[{Liu et~al.(2020)Liu, Arnon, Lazarus, Strong, Barrett, Kochenderfer
  et~al.}]{Liu_Arnon_Lazarus_Barrett_Kochenderfer_2019}
Liu, C.; Arnon, T.; Lazarus, C.; Strong, C.; Barrett, C.; Kochenderfer, M.~J.;
  et~al. 2020.
\newblock Algorithms for Verifying Deep Neural Networks.
\newblock \emph{Foundations and Trends{\textregistered} in Optimization}, 4.

\bibitem[{Liu and Tomizuka(2017)}]{liu2017designing}
Liu, C.; and Tomizuka, M. 2017.
\newblock Designing the robot behavior for safe human--robot interactions.
\newblock \emph{Trends in Control and Decision-Making for Human--Robot
  Collaboration Systems}, 241--270.

\bibitem[{Meng et~al.(2017)Meng, Liu, Cai, Han, and Tong}]{meng2017identity}
Meng, Z.; Liu, P.; Cai, J.; Han, S.; and Tong, Y. 2017.
\newblock Identity-aware convolutional neural network for facial expression
  recognition.
\newblock In \emph{2017 12th IEEE International Conference on Automatic Face \&
  Gesture Recognition (FG 2017)}, 558--565. IEEE.

\bibitem[{Ng et~al.(2006)Ng, Coates, Diel, Ganapathi, Schulte, Tse, Berger, and
  Liang}]{ng2006autonomous}
Ng, A.~Y.; Coates, A.; Diel, M.; Ganapathi, V.; Schulte, J.; Tse, B.; Berger,
  E.; and Liang, E. 2006.
\newblock Autonomous inverted helicopter flight via reinforcement learning.
\newblock In \emph{Experimental robotics IX}, 363--372. Springer.

\bibitem[{Prabhakar and Rahimi~Afzal(2019)}]{prabhakar2020abstraction}
Prabhakar, P.; and Rahimi~Afzal, Z. 2019.
\newblock Abstraction based Output Range Analysis for Neural Networks.
\newblock In \emph{Advances in Neural Information Processing Systems},
  volume~32. Curran Associates, Inc.

\bibitem[{Sezer, Ozbayoglu, and Dogdu(2017)}]{sezer2017deep}
Sezer, O.~B.; Ozbayoglu, M.; and Dogdu, E. 2017.
\newblock A deep neural-network based stock trading system based on
  evolutionary optimized technical analysis parameters.
\newblock \emph{Procedia computer science}, 114: 473--480.

\bibitem[{Si, Wei, and Liu(2019)}]{si2019agen}
Si, W.; Wei, T.; and Liu, C. 2019.
\newblock Agen: Adaptable generative prediction networks for autonomous
  driving.
\newblock In \emph{2019 IEEE Intelligent Vehicles Symposium (IV)}, 281--286.
  IEEE.

\bibitem[{Tjeng, Xiao, and Tedrake(2019)}]{Tjeng_Xiao_Tedrake_2019}
Tjeng, V.; Xiao, K.; and Tedrake, R. 2019.
\newblock Evaluating Robustness of Neural Networks with Mixed Integer
  Programming.
\newblock \emph{arXiv:1711.07356 [cs]}.
\newblock ArXiv: 1711.07356.

\bibitem[{Tran et~al.(2020)Tran, Yang, Lopez, Musau, Nguyen, Xiang, Bak, and
  Johnson}]{tran2020nnv}
Tran, H.-D.; Yang, X.; Lopez, D.~M.; Musau, P.; Nguyen, L.~V.; Xiang, W.; Bak,
  S.; and Johnson, T.~T. 2020.
\newblock NNV: The neural network verification tool for deep neural networks
  and learning-enabled cyber-physical systems.
\newblock In \emph{International Conference on Computer Aided Verification},
  3--17. Springer.

\bibitem[{Wang et~al.(2018{\natexlab{a}})Wang, Pei, Whitehouse, Yang, and
  Jana}]{wang2018efficient}
Wang, S.; Pei, K.; Whitehouse, J.; Yang, J.; and Jana, S. 2018{\natexlab{a}}.
\newblock Efficient Formal Safety Analysis of Neural Networks.
\newblock In \emph{Advances in Neural Information Processing Systems},
  volume~31. Curran Associates, Inc.

\bibitem[{Wang et~al.(2018{\natexlab{b}})Wang, Pei, Whitehouse, Yang, and
  Jana}]{wang2018formal}
Wang, S.; Pei, K.; Whitehouse, J.; Yang, J.; and Jana, S. 2018{\natexlab{b}}.
\newblock Formal security analysis of neural networks using symbolic intervals.
\newblock In \emph{27th $\{$USENIX$\}$ Security Symposium ($\{$USENIX$\}$
  Security 18)}, 1599--1614.

\bibitem[{Wong and Kolter(2018)}]{wong2018provable}
Wong, E.; and Kolter, Z. 2018.
\newblock Provable defenses against adversarial examples via the convex outer
  adversarial polytope.
\newblock In \emph{International conference on machine learning}, 5286--5295.
  PMLR.

\bibitem[{Xiang et~al.(2020)Xiang, Tran, Yang, and
  Johnson}]{Xiang_Tran_Yang_Johnson_2020}
Xiang, W.; Tran, H.-D.; Yang, X.; and Johnson, T.~T. 2020.
\newblock Reachable Set Estimation for Neural Network Control Systems: A
  Simulation-Guided Approach.
\newblock \emph{IEEE Transactions on Neural Networks and Learning Systems},
  1–10.

\end{thebibliography}
\newpage
\appendix 

\section{Appendix}

\subsection{Scalability analysis}           
    \label{appendix: scalability}
        For BMI, the performance does not change significantly with depth, width, branch, and changing dimensions because the acceleration mostly comes from saving the branch division. However, when the input change rate is high, the assumption that the branch division does not change very fast no longer holds, and its performance drops.
        
        For LB, its performance drops with the depth and width because the Lipschitz constant of the network grows with the size of the network. And it also performs worse when the input changes at a faster rate. LB does not change significantly with the dimension because its performance only depends on the norm of the input change, which does not grow significantly with the number of dimensions.
        
        RSR is similar to LB but is more sensitive to changes. Its performance drops with the size of the network because the over-approximation grows. And it also performs worse when the number of changing dimensions and changing rates increase. 
        
        BMW does not change significantly with any variables. Because we choose to construct the branches by partitioning the input set based on its shape, such that the branch construction only depends on the input set, not on the network.
        
        INN does change with the depth because the front layers usually have larger changes in our experiments as shown in \cref{fig: layer intervals}. Therefore, the over-approximation introduced by INN is mostly determined by front layers. And for width, because we consider the maximal change of a layer, the width does not make a big difference. It performs worse when the weight change increases, which directly enlarges $\bv \Delta_{layer}^*$.

        \begin{figure}[htb]
            \centering
            \includegraphics[width=1\linewidth]{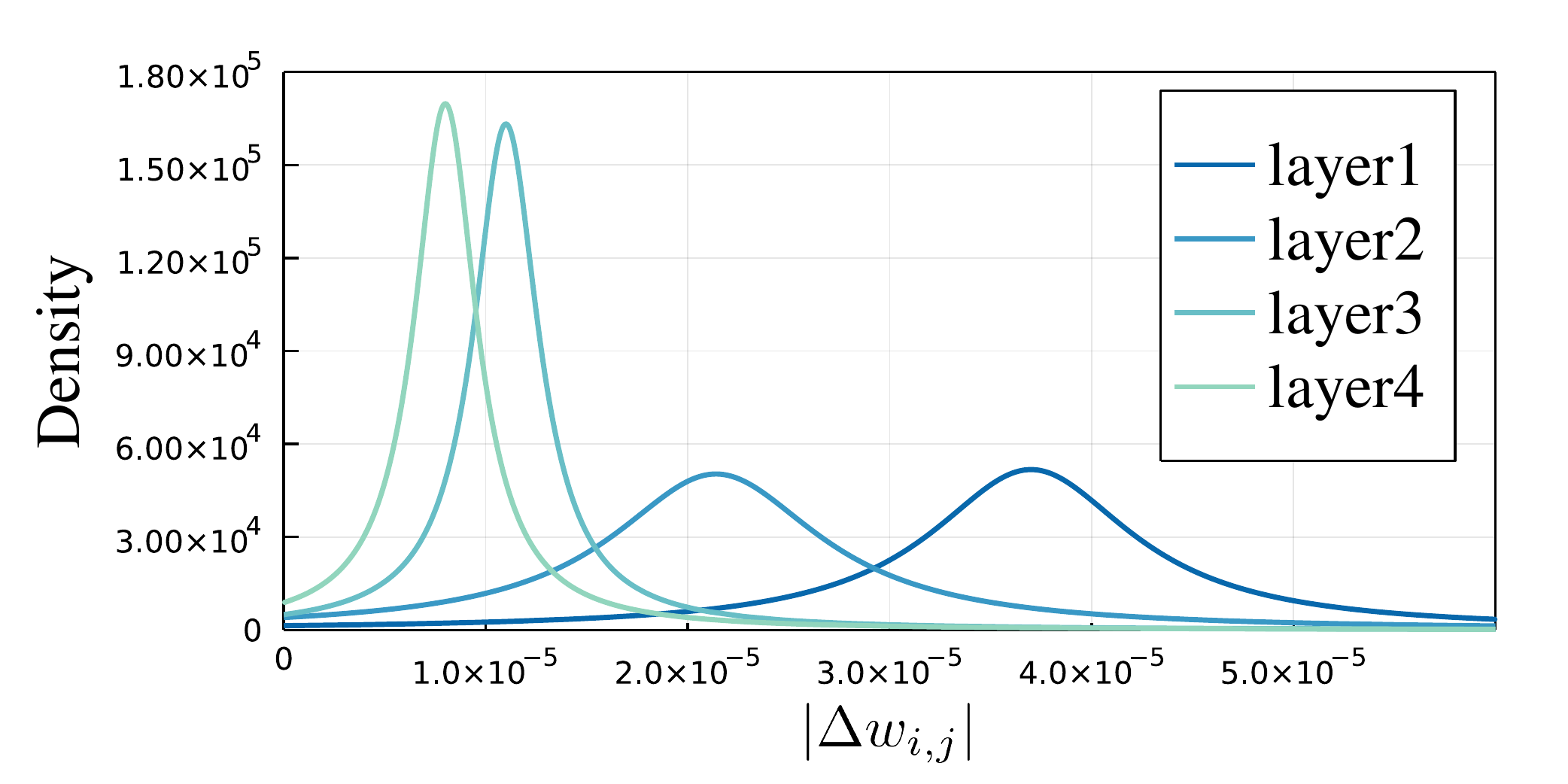}
            \caption{Weight change distribution of each layer in the RL task.}
            \label{fig: layer intervals}
        \end{figure}

    \subsection{RSR real-time computation guarantee}\label{appendix: RSR-real-time}
        We provide a real-time computation guarantee for the case when only $\bv b_i$ changes. Denote the input change by $\Delta \bv b^s_i(t) : = \bv b^s_i(t) - \bv b_i(t-1)$.
        \begin{lem} \label{lem: RSR-real-time} (RSR real-time computation guarantee)
        Suppose the constructing time of relaxed reachable sets is $T$, and the input changing gap is $\Delta t$. If only $\bv b_i$ changes and the input change $|\Delta \bv b^s_i(t) | \leq \bv \mu$ (dimension-wise comparison). Then $k = \lceil \frac{\min_i(\Delta_i/\bv \mu_i) \mu - T}{\Delta t} \rceil$ asynchronized processes can achieve real-time construction of relaxed reachable sets, where $i$ is the index of the vector element.
        \end{lem}
        \begin{proof}
        We can start a process every $T/k$ seconds to cover input changes from $[t+T, t+\min_i(\Delta_i/\bv \mu_i)]$. After the first $T$ seconds. There is always a relaxed reachable set that covers the current input set.
        \end{proof}

    \subsection{INN real-time computation} \label{appendix: INN-real-time}
    Suppose we relax the original weights $\bv W$ by an offset $\Delta$ at time $t_0$, \ie ${\underline{\bv W}(t_0)} = \bv W(t_0) - \Delta$, ${\overline{\bv W}(t_0)} = \bv W(t_0) + \Delta$. We define $\Delta \bv W(t) = \bv W(t) - \bv W(t-1)$.
        \begin{lem} (INN real-time computation guarantee)
        Suppose the constructing time of INN is $T$, and the weight changing interval is $\Delta t$. Assume that $|\Delta \bv W(t) | \leq \bv \mu$ (dimension-wise comparison). Then $k = \lceil \frac{\min_i(\Delta_i/\bv \mu_i) - T}{\Delta t} \rceil$ asynchronized processes can achieve real-time construction of INNs, where $i$ is the index of the vector element.
        \end{lem}
        \begin{proof}
        Similar to \cref{lem: RSR-real-time}, omitted.
        \end{proof}
        
    \subsection{Lipschitz Bound lemma and proof}

        We define $\Delta_{in}(t_0,t)$ as the maximum distance from $\XX_i(t)$ to $\XX_i(t_0)$: 
        \begin{align}
        \Delta_{in}(t_0, t) = \max_{\bv x' \in \XX_i(t)} \min_{\bv x\in \XX_i(t_0)} \| \bv x'- \bv x \|. \label{eq: delta_in_linear}
        \end{align}
        And with the Lipschitz constant $L$ of the network, we have $\forall x'\in\XX_i(t), \exists x\in\XX_i(t_0)$,
        \begin{align}
        \|\bv f (\bv x') - \bv f(\bv x)\| \leq L\ \| \bv x'- \bv x \| \leq L\ \cdot \Delta_{in}(t_0, t).
        \end{align}
        
        Suppose the output constraint is a linear function of the network output: 
        \begin{align}
            \YY: = \{ \bv y \mid \bv C \bv y \leq \bv d\}.
        \end{align}
        
        \begin{lem}(Lipschitz Tolerance)
        $\bv f(\XX_i(t)) \subseteq \YY$ if 
        \begin{align}
            \Delta_{in}(t_0, t) \leq \min_j \min_{\bv y \in \OO(\XX_i(t_0),\bv f)} \frac{ d_j - \bv c_j^T\ \bv y}{\|\bv c_j^T\|\ L}.
        \end{align}
        where $\bv c_j$ is the $j^{th}$ row of $\bv C$, and $d_j$ is the $j^{th}$ value of $\bv d$.
        \end{lem}
        
    \begin{proof}\label{sec: proof-lem}
        $\bv f(\XX_i(t_0)) \subseteq \YY$, that is
        \begin{align}
            \bv C\bv f (\bv x) \leq \bv d,\ \forall x\in \XX_i(t_0).
        \end{align}
        
        We want to verify that $\bv f(\XX_i(t)) \subseteq \YY$, that is
        \begin{align}
            \bv C\bv f (\bv x') \leq \bv d,\ \forall x'\in \XX_i(t). \label{eq: lip_con}
        \end{align}
        
        According to \cref{eq: Delta-in}$, \forall \bv x'$, we can find a $\bv x$ such that $\|\bv x' - \bv x\| \leq \Delta_{in}(t_0, t)$, for a row $\bv c_j$ in $\bv C$, $d_j$ in $\bv d$
        \begin{align}
            \bv c_j^T\ \bv f (\bv x') - d_j &= \bv c_j^T\ \bv f (\bv x') - \bv c_j^T\ \bv f(\bv x) + \bv c_j^T\ \bv f(\bv x) - d_j \\
            &\leq  \bv c_j^T\ [\bv f (\bv x') -  \bv f(\bv x)] + \bv c_j^T\ \bv f(\bv x) - d_j \\
            &\leq  \|\bv c_j^T\|\ \|\bv f (\bv x') - \bv f(\bv x)\| + \bv c_j^T\ \bv f(\bv x) - d_j \\
            &\leq  \|\bv c_j^T\|\ L \cdot \Delta_{in}(t_0, t) + \bv c_j^T\ \bv f(\bv x) - d_j. \label{eq: lip_2}
        \end{align}
        Therefore, a sufficient condition for \cref{eq: lip_con} is 
        \begin{align}
            & \max_{\bv x \in \XX_i(t_0)}\|\bv c_j^T\|\ L \cdot \Delta_{in}(t_0, t) + \bv c_j^T\ \bv f(\bv x) - d_j \leq 0,\ \forall j \\
            & \iff \Delta_{in}(t_0, t) \leq \min_j \min_{\bv x \in \XX_i(t_0)} \frac{d_j - \bv c_j^T\ \bv f(\bv x)}{\|\bv c_j^T\|\ L}
        \end{align}
        It's difficult to compute $\min_{\bv x \in \XX_i(t_0)} d_j - \bv c_j^T\ \bv f(\bv x)$, so we relax the inequality to 
        \begin{align}
            \Delta_{in}(t_0, t) &\leq \min_j \min_{\bv y \in \OO(\XX_i(t_0),\bv f)} \frac{ d_j - \bv c_j^T\ \bv y}{\|\bv c_j^T\|\ L}\\
            &\leq \min_j \min_{\bv x \in \XX_i(t_0)} \frac{d_j - \bv c_j^T\ \bv f(\bv x)}{\|\bv c_j^T\|\ L}
        \end{align}
        Then 
        \begin{align}
            \Delta_{in}(t_0, t) &\leq \min_j \min_{\bv y \in \OO(\XX_i(t_0),\bv f)} \frac{ d_j - \bv c_j^T\ \bv y}{\|\bv c_j^T\|\ L} \\
            &\implies \bv f(\XX_i(t)) \subseteq \YY
        \end{align}
    \end{proof}

    \subsection{Computation of $\Delta_{in}(t_0, t)$}
    
    \begin{lem}
    When $\XX_i(t)$ and $\XX_i(t_0)$ are both defined by linear constraints. Solving $\Delta_{in}(t_0, t)$ can be reduced to finite convex optimizations: 
    \begin{align}
        \Delta_{in}(t_0, t) = \max_{\bv x' \in \XX_i(t)} \min_{\bv x\in \XX_i(t_0)} \| \bv x'- \bv x \|.
    \end{align}
    \end{lem}
    \begin{proof}
        We first prove that 
        \begin{align}
             \max_{\bv x' \in \XX_i(t)} \min_{\bv x\in \XX_i(t_0)} \| \bv x'- \bv x \| = \max_{\bv x' \in \{\bv x_j\}} \min_{\bv x\in \XX_i(t_0)} \| \bv x'- \bv x \|,
        \end{align}
        where $\{\bv x_j\}$ are the vertices of $\XX_i(t)$. We prove by contradiction. If $\exists \bv x^* \notin \{\bv x_j\}$, $\forall \bv x_j$
        $$
        \min_{\bv x\in \XX_i(t_0)} \| \bv x^*- \bv x \| > \min_{\bv x\in \XX_i(t_0)} \| \bv x_j - \bv x \|
        $$
        We denote the solutions by $\bv x^*_0$ and $\bv x^j_0$. Then we have
        \begin{align}
            \| \bv x^*- \bv x^*_0 \| > \| \bv x_j - \bv x^j_0\| \label{eq: con_asp}
        \end{align}
        Because $\XX_i(t)$ is a polytope, there exists a set $\{\alpha_j\}$ such that $0 \leq \alpha_j \leq 1$, $\sum_j \alpha_j = 1$ and $\bv x^* = \sum_j \alpha_j \bv x_j$. Then
        \begin{align}
            & \| \bv x^*- \bv x^*_0 \| = \|  \sum_j \alpha_j \bv x_j -  (\sum_j \alpha_j ) \bv x^*_0 \|\\
            = & \| \sum_j \alpha_j (\bv x_j - \bv x^*_0) \| \leq  \sum_j \alpha_j \| \bv x_j - \bv x^*_0 \|\\
            \leq & \sum_j \alpha_j \| \bv x_j - \bv x^j_0 \| \leq \sum_j \alpha_j \max_j \| \bv x_j - \bv x^j_0 \|\\
            \leq & \max_j  \| \bv x_j - \bv x^j_0 \|
        \end{align}
    Contradicts with \cref{eq: con_asp}.
    \end{proof}
    
\subsection{Trade off analysis} \label{appendix: trade-off}
     \begin{figure*}[ht]
            \centering
            \includegraphics[width=.7\linewidth]{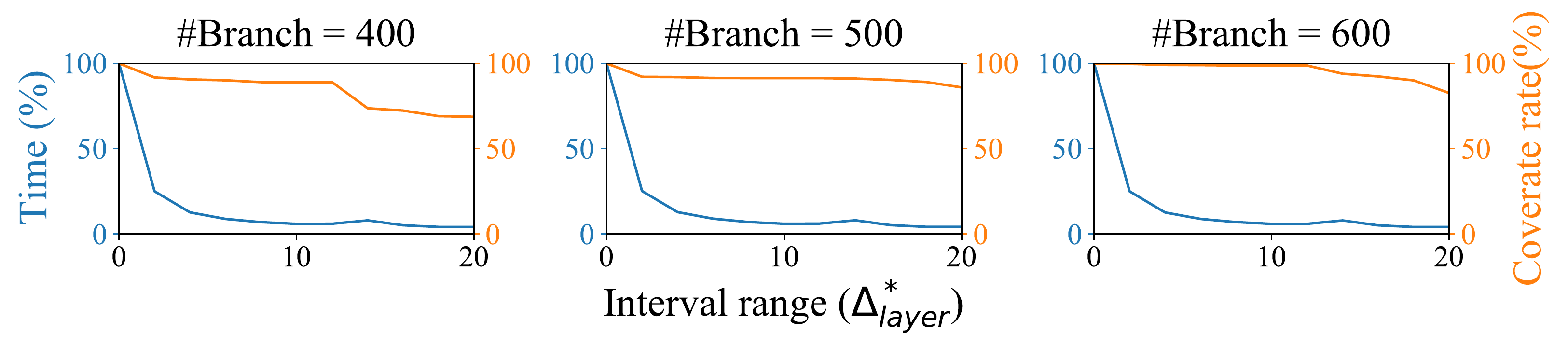}
            \caption{INN trade-off curves with different numbers of branches. The curve shows how time and coverage rate drops in percentage with the interval range increase. 
            }
            \label{fig: INN trade off}
        \end{figure*}
        \begin{figure*}[ht]
            \centering
            \includegraphics[width=.7\linewidth]{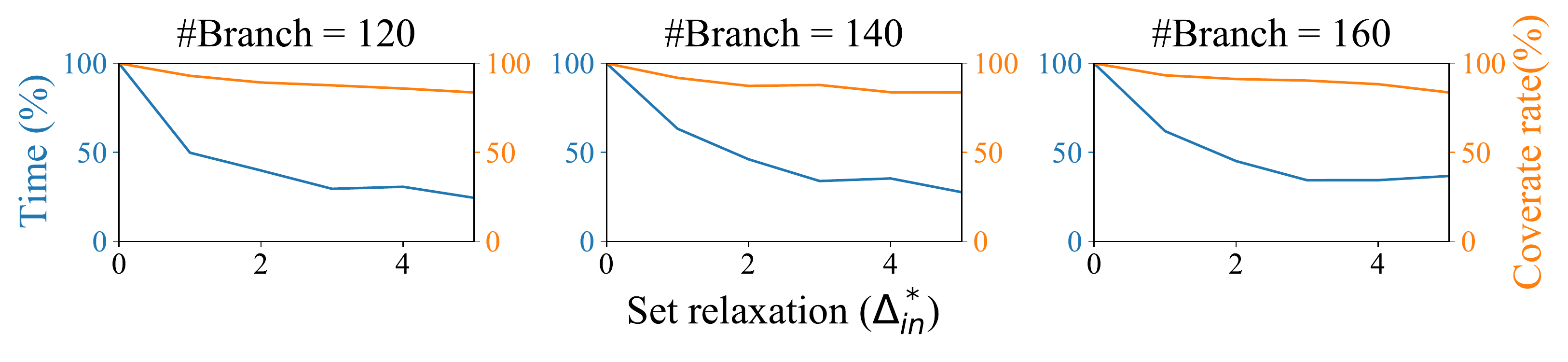}
            \caption{RSR trade-off curves with different numbers of branches. The results show that the coverage rate drops about linearly, and the time drops like a negative log function.}
            \label{fig: RSR trade off}
        \end{figure*}

        The over-approximation and speed of RSR and INN also depend on the number of branches. When there are more branches, the measure of each branch is smaller, more ReLU nodes have determined activation status, and the nonlinearity of the network is reduced therefore less over-approximation. Then we may be able to use a larger offset or interval. In this section, we draw additional trade-off curves to show the generalizability of the best-balanced trade-off point.

         \subsubsection{INN analysis}
        For INN. we test how the coverage rate changes with the interval range under three different numbers of branches. We choose the interval range as $[0, 20] \times \bv \Delta_{layer}^*$. And we choose the number of branches $400, 500, 600$, corresponding to different levels of verification accuracy. The coverage rate without INN is $70\%, 90\%, 100\%$.
        
        The trade-off curves are shown in \cref{fig: INN trade off}. We can see that, in all three sub-figures, when the interval range is at around $5\times \bv \Delta_{layer}^*$, the coverage rate has only an unnoticeable drop while the verification time reduces a lot. We consider $5\times \bv \Delta_{layer}^*$ is a well-balanced point with good generalizability.

        \subsubsection{RSR analysis}
        
        Similar to INN, the accuracy and speed of RSR depend both on the relaxation offset and the number of branches. We set the relaxation magnitude range as $[0, \Delta_{in}^* \times 5]$ because the time does not show significant change for larger relaxation. And we choose the branches $120, 140, 160$, corresponds coverage rate without set relaxation is $50\%, 66\%, 86\%$.
        
        We draw the trade-off curves to show how time and coverage rate drop in percentage as the set relaxation magnitude increases, as shown in \cref{fig: RSR trade off}. Time drops like a log function, and the coverage rate drops about linearly. The user can choose the relaxation magnitude based on the conditions. The chosen balanced point generalizes to other situations.

        \subsection{Algorithm details and pseudocode} \label{appendix: pseudocode}
        
        \subsubsection{BMI} \label{algo: BMI}
        The pseudocode is shown in \cref{pseudo: bmi}.
        \begin{algorithm}
        	\caption{BMI}
        	\small
        	\begin{algorithmic}[1]
        	\Function{BranchManage}{$\XX(t)$, $\bv f^t$, $X_S(t-1)$}
        	    \If{\Call{Coverage Rate}{$X_S(t-1)$} $<$ threshold}
        	    \State $X_S(t) \leftarrow$ \Call{reach + branch}{$\XX(t)$, $\bv f^t$, $\YY$} 
        	    \State tag $\leftarrow$ \{recompute\}
        	    \Else
                        \State $X_S(t) \leftarrow X_S(t-1)$ 
            	    \For {$(i,\XX_i(t))$ in enumerate($X_S(t)$)}
            	    \State Update base constraints for $\XX_i(t)$
            	    \If{$\XX_i(t) \subseteq \XX_i(t-1)$}
            	        \State $\text{tag}_i\leftarrow$ reuse
            	    \Else
            	        \State $\text{tag}_i\leftarrow$ recompute
            	    \EndIf
            	    \EndFor
            	\EndIf
        	    \State \Return $X_S(t), \text{tag}$
            \EndFunction
        	\end{algorithmic} \label{pseudo: bmi}
        \end{algorithm}
        
        \subsubsection{RSR} \label{algo: rsr}

        Specifically, for a branch 
        \begin{align}
            \XX_i(t_0) = \{x \mid \bv A_i(t_0) \bv x \leq \bv b_i(t_0)\},
        \end{align}
        we add an offset $\bv \Delta$ to create an enlarged input set: 
        \begin{align}
            \hat\XX_i(t_0) = \{x \mid \bv A_i(t_0) \bv x \leq \bv b_i(t_0) + \bv \Delta\}.
        \end{align}
        Then if the corresponding relaxed reachable set $\hat \OO_i(t_0) \subseteq \YY$. Then we can skip the computation for future $\XX_i(t)$ if $\XX_i(t) \subseteq \hat\XX_i(t_0)$, that is, when,
        \begin{align}
            \max_{\bv x\in\XX_i(t)} \bv A_i(t_0) \bv x - \bv b_i(t_0) - \bv \Delta \leq 0
        \end{align}

        
        The pseudo code is shown in \cref{pseudo: rsr1} and \cref{pseudo: rsr2}

        \begin{algorithm}
        	\caption{RSR Perturbation Analysis}
        	\small
        	\begin{algorithmic}[1]
        	\Function{Tolerable}{$t, t_0$}
        	    \If{$\max_{\bv x\in\XX_i(t)} \bv A_i(t_0) \bv x - \bv b_i(t_0) - \bv \Delta \leq 0$}
        	       \State \Return True    
                    \Else
        	       \State \Return False
        	   \EndIf
            \EndFunction
        	\end{algorithmic} \label{pseudo: rsr1}
        \end{algorithm}
        
        \begin{algorithm}
        	\caption{RSR Reach}
        	\small
        	\begin{algorithmic}[1]
        	\Function{Reach}{$\XX(t),\bv f^t$}
        	    \State $R_{0}(t) \leftarrow \XX(t)$
        	    \For{j=1 to $m_x$}
        	        \State $R_0(t).b_j += c_j $\Comment{Add an offset to the constant term.}
        	    \EndFor
    	        \For{k=1 to n}
    	             \State $R_{k}(t) \leftarrow$ \Call{Reach}{$(R_{k-1}(t), \bv W_{k}(t)$}
    	        \EndFor
    	        \State \Return $R_{n}(t)$
            \EndFunction
        	\end{algorithmic} \label{pseudo: rsr2}
        \end{algorithm}
        
        \subsubsection{LB} \label{algo: lb}
        Suppose the output constraints are
        \begin{align}
            \YY &= \{y \mid a_j y \leq b_j, j=1\dots m_y\}
        \end{align}
        The pseudo code of Lipschitz bound method is shown in \cref{pseudo: lb}
        \begin{algorithm}
        	\caption{LB}
        	\small
        	\begin{algorithmic}[1]
        	\Function{Tolerable}{$t, t_0$}
        	    \State $d \leftarrow \Delta_{in}(t_0,t)$ \Comment{Linear programming.}
        	    \State $\delta \leftarrow \min_j \min_{\bv y \in \OO(\XX_i(t_0),\bv f)} \frac{ d_j - \bv c_j^T\bv y}{\|\bv c_j^T\|\ L}.$ \Comment{Can be obtained from constraints check of the previous round}
        	    \State \Return $d \leq \delta$
            \EndFunction
        	\end{algorithmic} \label{pseudo: lb}
        \end{algorithm}
        
        \subsubsection{BMW} \label{algo: bmw}
        The pseudo code for BMW is shown in \cref{pseudo: bmw1}.
        
        \begin{algorithm}
        	\caption{BMW}
        	\small
        	\begin{algorithmic}[1]
        	\Function{BranchManage}{$\XX(t),\bv f^t, X_S(t-1)$}
                    \If{\Call{Coverage Rate}{$X_S(t-1)$} $<$ threshold}
        	    \State $X_S(t) \leftarrow$ \Call{reach + branch}{$\XX(t)$, $\bv f^t$, $\YY$} 
        	    \Else
                        \State $X_S(t) \leftarrow X_S(t-1)$ 
            	\EndIf
                    \State tag $\leftarrow$ \{recompute\}
        	    \State \Return $X_S(t), \text{tag}$
            \EndFunction
        	\end{algorithmic} \label{pseudo: bmw1}
        \end{algorithm}

        	        

        
        \subsubsection{INN} \label{algo: inn}
        We denote the weight range of the interval network by $W_{int}$. You can find how to build the interval network in \cite{prabhakar2020abstraction}. 
        We denote the $k^{th}$ layer of the reachable set at time step $t$ by $R_k(t)$. The pseudo code is shown in \cref{pseudo: inn1} and \cref{pseudo: inn2}.
        \begin{algorithm}
        	\caption{INN Perturbation Analysis}
        	\small
        	\begin{algorithmic}[1]
        	\Function{Tolerable}{$t, t_0$}
        	    \State \Return $\bv W(t) \in \bv W^{int}(t_0)$
            \EndFunction
        	\end{algorithmic} \label{pseudo: inn1}
        \end{algorithm}
        
        \begin{algorithm}
        	\caption{INN Reach}
        	\small
        	\begin{algorithmic}[1]
        	\Function{Reach}{$\XX(t),\bv f^t$}
        	    \State $R_{0}(t) \leftarrow \XX(t)$
        	    \State $\bv W^{int}(t)\leftarrow$\Call{build interval net}{$\bv W(t)$}
    	        \For{k=1 to n}
    	             \State $R_{k}(t) \leftarrow$ \Call{Reach}{$(R_{k-1}(t), \bv W^{int}_{k}(t)$}
    	        \EndFor
    	        \State \Return $R_{n}(t)$
            \EndFunction
        	\end{algorithmic} \label{pseudo: inn2}
        \end{algorithm}

        \subsubsection{IC} \label{algo: ic}
        The pseudo-code of IC is shown in \cref{pseudo: ic}.
        \begin{algorithm}
        	\caption{IC}
        	\small
        	\begin{algorithmic}[1]
        	\Function{Reach}{$\XX(t),\bv f^t$}
        	    \If{ $\XX(t-1) = \XX(t)$ and only the last layer weights $W_n(t)$ changes}
        	        \State $R_{n-1}(t) \leftarrow R_{n-1}(t-1)$
        	        \State \Return \Call{Reach}{$(R_{n-1}(t), \bv W_n(t)$}
        	    \Else
        	        \State $R_{0}(t) \leftarrow \XX(t)$
        	        \For{k=1 to n}
        	             \State $R_{k}(t) \leftarrow$ \Call{Reach}{$(R_{k-1}(t), \bv W_{k}(t)$}
        	        \EndFor
        	        \State \Return $R_{n}(t)$
        	    \EndIf
            \EndFunction
        	\end{algorithmic} \label{pseudo: ic}
        \end{algorithm}

\end{document}